\documentclass{article}

\usepackage{amsthm}              
\usepackage{hyperref}
\usepackage{makeidx}         
\usepackage{graphicx}        
\usepackage{multicol}        
\usepackage{multirow}
\usepackage[bottom]{footmisc}
\usepackage{comment}
\usepackage{booktabs}
\usepackage[margin=3cm]{geometry}

\usepackage{newtxtext}       %
\usepackage[varvw]{newtxmath}       
\usepackage{pgffor}

\usepackage{subfigure}

\usepackage[capitalize]{cleveref}

\foreach \a in {q,w,e,r,t,y,u,i,o,p,a,s,d,f,g,h,j,k,l,z,x,c,v,b,n,m,Q,W,E,R,T,Y,U,I,O,P,A,S,D,F,G,H,J,K,L,Z,X,C,V,B,N,M}{  \expandafter\xdef\csname v\a\endcsname {{\noexpand\mathbf{\a}}}}

\newcommand{\vtheta}{\boldsymbol{\theta}}
\newcommand{\Id}{\vI_{\rm d}}

\newcommand{\RR}{\mathbb{R}}

\newcommand{\lp}{\left(}
\newcommand{\rp}{\right)}

\newcommand{\sign}{\ensuremath{\text{\rm sign}}}
\newcommand{\minimize}[2]{\ensuremath{\underset{\substack{{#1}}}%
{\text{\rm minimize}}\;\;#2 }}
\newcommand{\rhinge}{\rho_{\rm hinge}}

\usepackage{xcolor}
\newcommand{\AB}[1]{\textcolor{black}{#1}}

\newcommand{\GF}[1]{\textcolor{black}{#1}}

\usepackage{algorithm}
\usepackage{algorithmic}

\usepackage{url}
\newtheorem{remark}{Remark}
\newtheorem{definition}{Definition}
\newtheorem{proposition}{Proposition}
\newtheorem{theorem}{Theorem}

\title{Majorization-Minimization for sparse SVMs}
\usepackage[affil-it]{authblk}
\author[1]{A. Benfenati}
\author[2]{E. Chouzenoux}
\author[3]{G. Franchini}
\author[4]{S. Latva--{\"A}ij{\"o}}
\author[5]{D. Narnhofer}
\author[6]{J.--C. Pesquet}
\author[7]{S. J. Scott}
\author[8]{M. Yousefi}
\affil[1]{Environmental and Science Department, Via Celoria 2, 20133, Milano, Italy.}
\affil[2,6]{CVN, Inria, CentraleSupélec, University Paris Saclay, 9 rue Joliot Curie, 91190, Gif-sur-Yvette}
\affil[3]{Department of Physics, Informatics and Mathematics, Via Campi 213/B, 41125, Modena, Italy.}
\affil[4]{Department of Mathematics and Statistics, Pietari Kalmin katu 5, 00014 Helsinki, Finland}
\affil[5]{Institute for Computer Graphic and Vision, Graz University of Technology, Inffeldgasse 16, 8010 Graz, Austria}
\affil[7]{Department of Mathematical Sciences, University of Bath, BA2 7AY, United Kingdom}
\affil[8]{Department of Mathematics and Geosciences, University of Trieste, via Valerio 12/1, 34127 TS, Italy}

\date{}
\begin{document}

\maketitle
\abstract{Several decades ago, Support Vector Machines (SVMs) were introduced for performing binary classification tasks, under a supervised framework. Nowadays, they often outperform other supervised methods and remain one of the most popular approaches in the machine learning arena. In this work, we investigate the training of SVMs through a smooth sparse-promoting-regularized squared hinge loss minimization. This choice paves the way to the application of quick training methods built on majorization-minimization approaches,  benefiting from the Lipschitz differentiabililty of the loss function. Moreover, the proposed approach allows us to handle sparsity-preserving regularizers promoting the selection of the most significant features, so enhancing the performance. Numerical tests and comparisons conducted on three different datasets demonstrate the good performance of the proposed methodology in terms of qualitative metrics (accuracy, precision, recall, and F$_1$ score) as well as computational cost.}

\section{Introduction}
\label{sec:intro}

Support Vector Machines (SVMs) are well-tailored for regression and classification applications. They were introduced in the seminal work \cite{Cortes95} for supervised learning. In addition to being grounded on sound optimization techniques \cite{platt1998sequential,vapnik1999nature}, various  extensions of them can be performed. They remain one of the most widely used methods in classification tasks despite the increasing \GF{role} played by neural networks. As linear classifiers, SVMs have been shown to outperform many supervised methods \cite{LIANG2017126,cervantes2015data,naik2017online}. Real-world applications include image classification \cite{7882747}, face detection \cite{tsai2018facial,ZHAN201619}, hand-written character recognition \cite{9339435}, melanoma classification \cite{afifi2019system,afifi2020dynamic}, text categorization \cite{pinheiro2019combining,7738833}. The interested reader can find a complete review in \cite{CERVANTES2020189}.

The supervised learning problem in the SVM framework consists in minimizing a suitable function measuring the distance between the predicted and the true labels corresponding to a dataset sample. This minimization is carried out with respect to the SVM model parameters. The SVM training problem may be formulated as a quadratic programming one \cite{zanghirati2003parallel}. It may involve least squares loss (hard--margin SVM) under a suitable constraint \cite{shalev2014understanding}, or the hinge loss \cite{7738833}.

For SVM training, the optimization problem can be solved via Lagrangian duality approaches \cite{Cortes95, 6976941}, naturally leading to some clever strategies such as kernel tricks \cite{5947214}, splitting the problem into simpler subproblems \cite{6976941}, cutting plane procedures \cite{CHU2017127}. In several applications, it is common to promote sparsity on the SVM parameters. This is equivalent to implicitly enforcing \textit{feature selection}, meaning that only the features that are essential to a task are kept, and those that are useless or even result in noisy solutions are dropped. A classification task with a limited or severely unbalanced dataset, facing the so-called overfitting problem, is a \GF{standard} scenario in which this sparsity condition is required. In such a case, \GF{standard} (nonsparse) SVMs-based models,  which are particularly tailored for specific datasets, might not be able in generalization to adapt properly to unseen data. Introducing regularization to achieve this property in the training loss is the most popular method for inducing sparsity on SVM parameters.

The most specifically designed functionals that impose sparsity on a solution are the so-called $\ell_0$-norm\footnote{Recall that actually this function is not a proper norm.} \cite{weston2003use}, and its well-known convex relaxation, i.e., $\ell_1$-norm \cite{bradley1998feature}. Other functionals,
including $\ell_{1,p}$ norms, or elastic-net functionals, have also been shown to effectively promote sparse regularization, see e.g.\cite{rosasco2013nonparametric,9447785}. To effectively accommodate the additional penalty term, the modification of a training algorithm must be considered. For example, the SVM-based objective function involving the hinge loss and a $\ell_{1,p}$-norm can be efficiently minimized by primal-dual methods \cite{7738833}.

The presented work focuses on training SVMs when we employ the squared hinge loss as a data fidelity function, coupled with a smooth sparsity-promoting regularization functional. In this way, as we will show hereafter, the loss function is Lipschitz differentiable. This paves the way to the application of fast training techniques. This work investigates first-order methods such as the gradient descent algorithm and discusses its acceleration via Majorization--Minimization (MM) techniques.

\emph{\textbf{Contribution.}} The focus of this work is on training an SVMs-based model using a smooth regularization functional that promotes sparsity and the squared hinge loss as a data fidelity function. This makes the loss function Lipschitz differentiable, as we demonstrate, and allows fast training techniques. Moreover, this work explores and analyzes how Majorization-Minimization (MM) strategies can speed up first-order methods.

\emph{\textbf{Outline.}} This paper is organized as follows. Together with the regularization functionals taken into consideration in this work, the problem is formulated in \Cref{sec:problem}. \Cref{sec:MP} presents the theoretical foundations of MM methods and \Cref{sec:MMalg} depicts the MM-based algorithms as well as their practical implementation in our context. The goal of \Cref{sec:numexp} is to assess the performance of the suggested procedures. Extensive comparisons are conducted with respect to state-of-the-art training algorithms. Finally, \Cref{sec:concl} draws conclusions.

\emph{\textbf{Notation.}} Bold letters denote vectors, while bold capital letters denote matrices. Greek and italic Latin letters refer to scalars. $\RR^n$ is the real Euclidean space of dimension $n$, $\RR^{m\times n}$ denotes the real space of matrices with $m$ rows and $n$ columns. For $\vx\in\RR^n$ $\|\vx\|$ denotes the classical Euclidean (or $\ell_2$) norm of the vector. For a matrix $\vA$, $\|\vA\|$ is the spectral norm of $\vA$. The function $\mathsf{1}_\Omega$ denotes the binary indicator of the set $\Omega$, $\mathsf{1}_{\Omega}(x)=1$ if $x\in\Omega$, 0 otherwise.

\section{Problem Formulation}
\label{sec:problem}
The problem addressed in this work is binary classification. Given a new observation $\vx\in\RR^n$, which contains the describing $n$ scalar features, the aim is to categorize $\vx$ into one of two classes. In this section, we present the two main components of the mathematical model we propose to solve this task: the SVM data fidelity term, namely here, the squared hinge loss, and the regularization functionals that promote sparsity. 

\subsection{SVM loss function}
The mathematical model for the categorization of the observation $\vx$ into two classes encompasses a linear classifier
\begin{equation}
\label{eq:linClass}
\begin{aligned}
M: \RR^n&\to \{-1,1\}\\
\vx &\to \sign(\vw^\top\vx+\beta)
\end{aligned}
\end{equation}
where $\vw\in\RR^n$ and $\beta\in\RR$. The classifier in \eqref{eq:linClass} defines a separating hyperplane whose purpose is to distinguish between items belonging to different classes, see e.g. \cref{fig:linclassEX} for a 2D example. The output of $M$ in \cref{eq:linClass} will be then $1$ or $-1$, and these two labels correspond to the categorization in one of the classes.
\begin{figure}[htbp]
	\centering
	\includegraphics[width=0.5\textwidth]{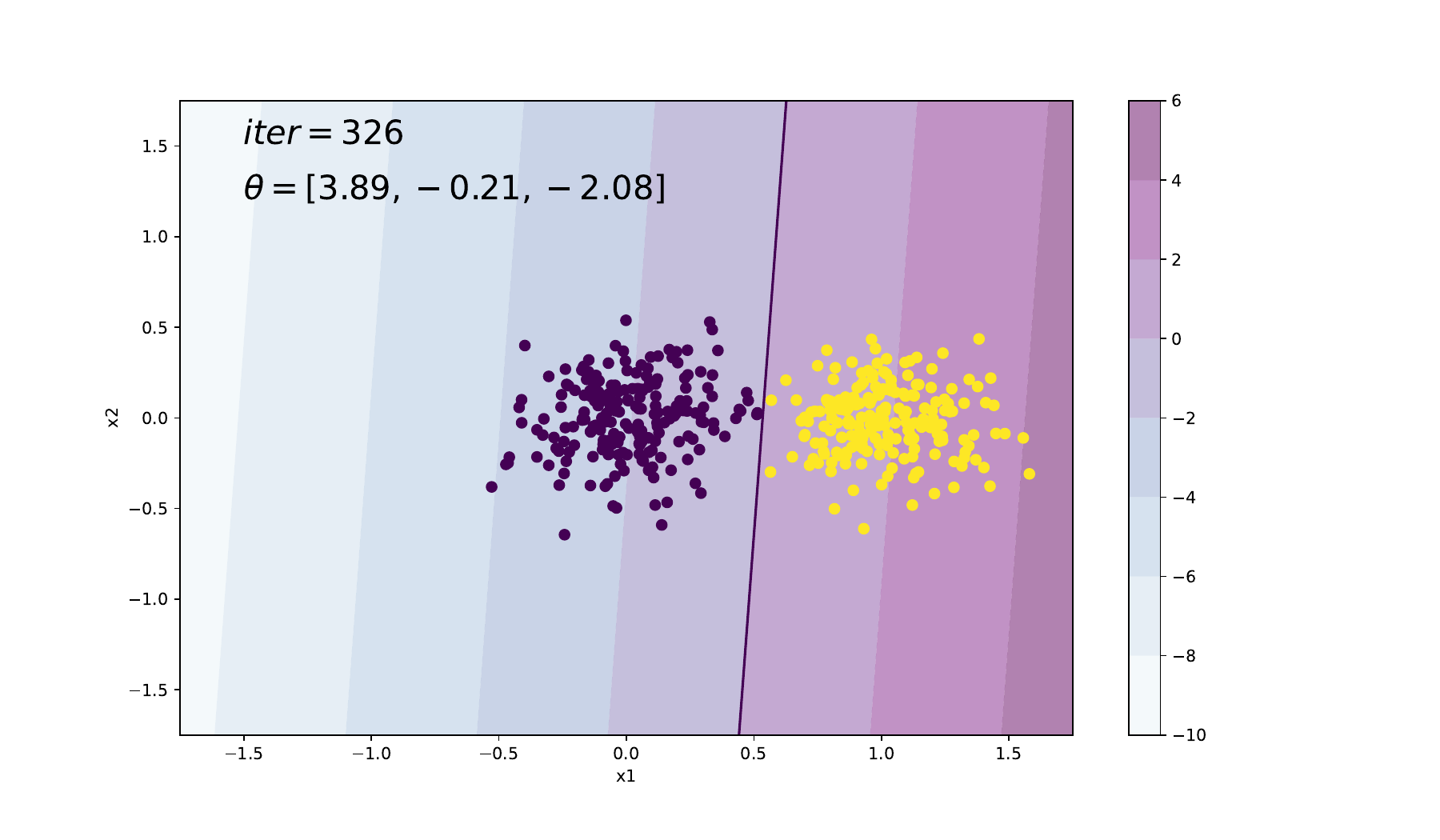}
	\caption{Toy example of a linear classifier. The line easily separates the two classes.}
	\label{fig:linclassEX}
\end{figure}
The parameters $\vw$ and $\beta$ in \eqref{eq:linClass} must be estimated during a training phase. Given a dataset $\{(\vx_k,y_k)\}_{k=1,\dots,K},\, \vx_k\in\RR^n,\, y_k\in\{-1,1\}$, and where $\vx_{k,i}$ denotes the $i$--th feature of the $k$--th sample, the ideal training loss would consist in the misclassification count
\begin{equation}
\label{eq:misclassCount}
\ell\lp M(\vx_{k}),y_{k}\rp =\frac{1-y_{k} M(\vx_{k})}{2} = \rho\lp y_{k}(\vw^\top \vx_{k}+\beta)\rp
\end{equation}
where
$$
(\forall \upsilon \in \RR)\qquad \rho(\upsilon) = \frac{1-\sign(\upsilon)}{2}.
$$
The training would then be carried on by solving the following optimization problem:
\begin{equation}
\label{eq:optPB}
\minimize{\vw\in\RR^n,\, \beta\in\RR}{\sum_{k=1}^K \rho \lp  y_{k}(\vw^\top \vx_{k}+\beta)\rp}.
\end{equation}
Unfortunately, \eqref{eq:optPB} reveals to be a difficult nonconvex problem. To overcome this issue, a popular choice is to subsitute the hinge loss $\rhinge$  for $\rho$ in \eqref{eq:misclassCount}, where
\begin{equation}
\label{eq:hingeLoss}
(\forall \upsilon \in \RR)\qquad \rhinge(\upsilon) = \max\{1-\upsilon,0\}.
\end{equation}

Function \eqref{eq:hingeLoss} provides the minimal convex upper bound of the misclassification count function, see \cref{fig:hingeLoss} for a visual inspection.
\begin{figure}
	\centering
	\includegraphics[width=0.6\textwidth]{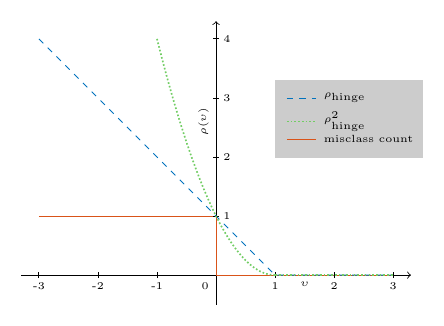}
	\caption{Hinge loss function (cyan) versus misclassification count function (orange) for the case in which the true label is 1. The hinge loss strongly penalizes uncorrected labels less than 0, while it assumes low values for outputs in $[0,1]$. Obviously, when the classifier provides the correct label, the loss is zero. }
	\label{fig:hingeLoss}
\end{figure}

\begin{remark}
	When the two classes are nonempty and separable, the goal of the training is to find a separating hyperplane such that
	$$
	\begin{cases}
	\vw^\top \vx_{k}+\beta > 0 & \mbox{if $y_{k}=1$}\\
	\vw^\top \vx_{k}+\beta < 0 & \mbox{if $y_{k}=-1$} 
	\end{cases}
	$$
	In the case of nonseparable classes, one should employ a slack variable:
	$$
	\begin{cases}
	\vw^\top \vx_{k}+\beta \ge 1-\xi_{k} & \mbox{if $y_{k}=1$}\\
	\vw^\top \vx_{k}+\beta \le -1+\xi_{k} & \mbox{if $y_{k}=-1$} 
	\end{cases}
	$$
	which has the following interpretation:
	$$
	\ell(M(\vx_{k}),y_{k})  = \min_{\xi_{k}\in [0,+\infty]} \xi_{k} \quad \text{s.t.}\quad y_{k}(\vw^\top \vx_{k}+\beta) \ge 1-\xi_{k}.
	$$
\end{remark}

The hinge loss in \eqref{eq:hingeLoss} is convex but not differentiable and it may cause numerical issues around $\upsilon=1$. The training can be performed by using primal--dual methods for solving the related optimization problem \cite{7738833}, which are usually costly and might lack flexibility. To overcome this issue, we focus instead on the squared hinge loss
\begin{equation}
\label{eq:squaredHingeLoss}
(\forall \upsilon \in \RR)\qquad \rhinge^2(\upsilon) = \max\{(1-\upsilon)^2,0\}.
\end{equation}
Function \eqref{eq:squaredHingeLoss} is convex and differentiable on the entire domain. Moreover, it has a 2--Lipschitz gradient, which is a useful property when dealing with optimization problems. A possible drawback is that the squared hinge loss might be more sensitive than the hinge function with respect to larger errors (see \Cref{fig:hingeLoss} for comparison).

\subsection{Regularization}
Our work focuses on the regularized version of the squared-hinge loss problem:
\begin{equation}
\label{eq:optProbSqHinge}
\minimize{\vw \in \RR^{N},\,\beta \in \RR}{\sum_{k=1}^K \rhinge^2\big(y_{k}(\vw^\top \vx_{k}+\beta)\big)+ f(\vw)}.
\end{equation}
Various choices can be made for function $f$, to favor the sparsity of the solution. Below, we list some examples covered by our approach.
Namely, we consider
\begin{equation}
\label{eq:fReg}
\lp\forall\, \vw = (w_i)_{1 \leq i \leq N} \in \mathbb{R}^N\rp \quad
f(\vw) = \sum_{i=1}^N \varphi(w_i) + \frac{\eta}{2} \| \vw\|^2,
\end{equation}
where $\varphi: \mathbb{R} \to \mathbb{R}$ is a potential function and $\eta \geq 0$, for which we introduce the following requirements:
\begin{enumerate}
	\item $\varphi$ is even ;
	\item $\varphi$ is differentiable on $\mathbb{R}$ ;
	\item $\varphi\lp\sqrt{\cdot}\rp$ is concave on $[0,+\infty[$.
\end{enumerate}
This framework is rather versatile, as it allows us to consider several interesting choices, such as smooth approximations for the $\ell_1$ norm  or for the $\ell_0$ pseudo--norm. For $\varphi \equiv 0$, we retrieve the standard quadratic penalty often used in SVMs. When $\eta\neq0$ and $\varphi$ is a sparse promoting term, $f$ can be viewed as an elastic-net penalty. See \cref{fig:regFunctions} for a visual inspection. Typically, we can use the so-called hyperbolic potential defined, for $\lambda \geq 0$, by
\begin{equation}
\label{eq:l1Smooth}
(\forall w \in \mathbb{R}) \quad
\varphi(w) = \lambda \sqrt{w^2+\delta^2}, \, \delta>0.
\end{equation}
Function~\eqref{eq:l1Smooth} is a convex function approximating $w \mapsto \lambda |w|$. Another choice is the Welsh potential
\begin{equation}
\label{eq:l0Smooth}
(\forall w \in \mathbb{R}) \quad
\varphi(w) = \lambda \lp 1-\exp\left(-\frac{w^2}{2 \delta^2}\right)\rp, \, \delta>0,
\end{equation}
Function \eqref{eq:l0Smooth} is nonconvex and approximates the binary indicator function
$$
w \mapsto \lambda \mathsf{1}_{w\neq 0}.
$$

\begin{figure}
	\centering
	\subfigure[]{\includegraphics[width=0.45\textwidth]{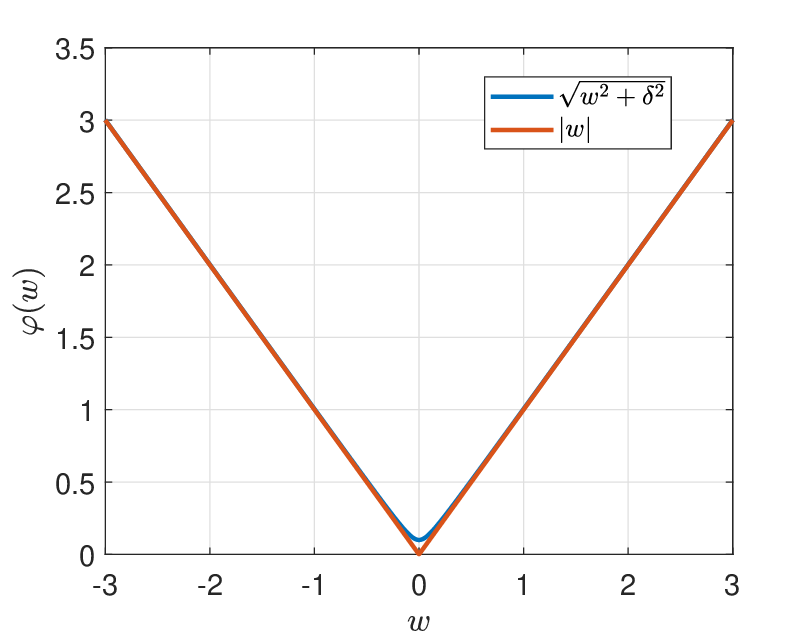}} \hfill 
	\subfigure[]{\includegraphics[width=0.45\textwidth]{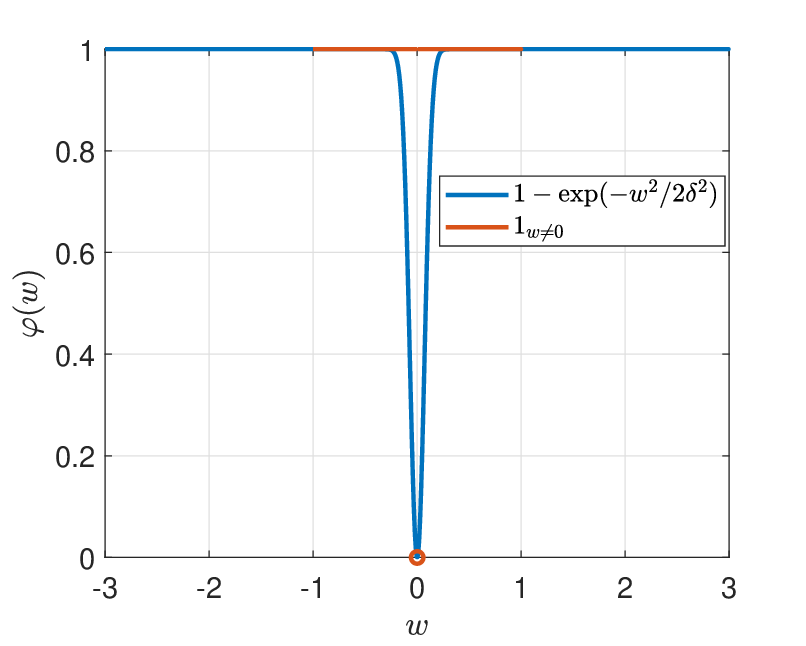}} 
	\caption{(a) Absolute value and its smooth approximation with hyperbolic potential. (b) Binary indicator function and its smooth approximation with Welsh potential.} 
	\label{fig:regFunctions}
\end{figure}

\subsection{General formulation}

Problem \eqref{eq:optProbSqHinge} can be reformulated as in 
the following way:

\begin{equation}
\label{eq:sqHingeRef}
\minimize{\vtheta \in \RR^{N+1}}{(g(\vL\vtheta)+\tilde{f}(\vtheta) \equiv \Phi(\vtheta))}
\end{equation}
where
\begin{itemize}
	\item $\vtheta = [\vw^\top\;\;\beta]^\top\in \RR^{N+1}$
	\item $\vL = \mathrm{Diag}(y_1,\cdots,y_K)  \begin{bmatrix}
	\vx_1^\top & 1\\
	\vdots & \vdots\\
	\vx_K^\top &  1
	\end{bmatrix} \in \RR^{K\times (N+1)}$
	\item[]
	\item $\lp\forall\, \vv = (v_k)_{1\le k \le K}\rp$
	$\displaystyle g(\vv) = \sum_{k=1}^K$ $\rhinge^2(v_k)$
	\item $\displaystyle \tilde{f}(\vtheta) = f(\vw)$.
\end{itemize}
Note that the regularization term only affects the variable $\vw$ and not the bias $\beta$.
Function $\Phi$ involved in \eqref{eq:optProbSqHinge} is differentiable on $\mathbb{R}^{N+1}$. Its gradient reads
\begin{equation}
\label{eq:gradient}
\lp\forall\, \vtheta \in \mathbb{R}^{N+1}\rp \quad 
\nabla \Phi(\vtheta) =  \vL^\top \nabla g(\vL\vtheta)
+ \nabla \tilde{f}(\vtheta).
\end{equation}
The derivative of the squared hinge loss is
\begin{equation}
\label{eq:gradientsqhinge}
(\forall \vv \in \mathbb{R}^K) \quad 
\nabla g(\vv) = \left( \max(2 (v_k - 1), 0) \right)_{1 \leq k \leq K}.
\end{equation}
Moreover, 
\begin{equation}
\label{eq:gradientreg}    (\forall \vtheta \in \mathbb{R}^{N+1}) \quad 
\nabla \tilde{f}(\vtheta) = \left(\begin{array}{c}\varphi'(w_1) + \eta w_1\\\vdots \\ \varphi'(w_N) + \eta w_N\\ 0 \end{array}  \right),
\end{equation}
with $\varphi'$ as the derivative of the potential function $\varphi$ involved in the construction of the regularization term $f$. In particular, for the hyperbolic potential function \eqref{eq:l1Smooth},
\begin{equation}
(\forall w \in \mathbb{R}) \quad \varphi'(w) = \lambda \frac{w}{\sqrt{w^2+ \delta^2}}.
\end{equation}
while, for the Welsh potential \eqref{eq:l0Smooth},
\begin{equation}
(\forall w \in \mathbb{R}) \quad
\varphi'(w) =\lambda \frac{w}{\delta^2} \exp\left(- \frac{w^2}{2 \delta^2}\right)
\end{equation}

In Section \ref{sec:MP}, we give some important additional properties of function $\Phi$. Then, in Section \ref{sec:MMalg}, we provide a family of algorithms based on the MM principle to solve Problem \eqref{eq:sqHingeRef}.

\section{Majorization Properties}
\label{sec:MP}
This section is devoted to presenting a key tool as the core of the training algorithms proposed in this work, namely the MM technique and the underlying concept of majorizing approximation. The MM technique consists of alternating between two steps to solve an initial complex optimization problem. The first step involves computing the tangent majorant of the objective function, and the second is to minimize that majorant in order to progressively converge to a reliable minimizer of the original function. The definition of a majorant function for the function $\Phi$ in \eqref{eq:sqHingeRef} is given in the following.

\begin{definition}
	\label{def:majorant}
	A tangent majorant $h(\cdot;\vtheta'):\RR^{N+1} \to\RR$ of $\Phi$ at $\vtheta' \in \RR^{N+1}$ is a function such that
	\begin{eqnarray*}
		&& h(\vtheta;\vtheta')\geq \Phi(\vtheta) \quad(\forall\, \vtheta\in \RR^{N+1})\\
		&& h(\vtheta';\vtheta')=\Phi(\vtheta')
	\end{eqnarray*}
\end{definition}
The general MM iterative scheme then reads:
\begin{equation}
(\forall n \in \mathbb{N}) \quad \vtheta^{(n+1)} = \text{argmin}_{\vtheta \in \RR^{N+1}} h(\vtheta;\vtheta^{(n)}),
\end{equation}
with some initialization $\vtheta^{(0)} \in \RR^{N+1}$. Under suitable hypotheses on the loss function in \eqref{eq:sqHingeRef} and its majorizing approximations, the iterative scheme leads to a sequence converging to a solution \cite{Jacobson07}. 

Let us now discuss the construction of  reliable majorizing approximations for the considered function $\Phi$.

\subsection{Descent lemma majorant}
\begin{proposition}
	\label{prp:sHingeLip}
	Assume that $\varphi$ is $a$-Lipschitz differentiable on $\mathbb{R}$, with $a>0$. Then, function $\Phi$ involved in \eqref{eq:sqHingeRef} is $\mu$-Lipschitz differentiable on $\mathbb{R}^{N+1}$ with
	\begin{equation}
	\mu = 2 \|\vL\|^2 + a + \eta.
	\end{equation}
	As a consequence, for every $\vtheta' \in \mathbb{R}^{N+1}$, the following function is a tangent majorant of $\Phi$ at $\vtheta'$,
	\begin{equation}
	(\forall \vtheta \in \mathbb{R}^{N+1}) \quad
	h(\vtheta,\vtheta') =   \Phi(\vtheta') + \nabla    \Phi(\vtheta')^\top (\vtheta - \vtheta') + \frac{\mu}{2} \|\vtheta - \vtheta'\|^2. 
	\end{equation}
\end{proposition}

{Note that functions  \eqref{eq:l1Smooth} and \eqref{eq:l0Smooth} are $a$-Lipschitz differentiable with $a = \frac{\lambda}{\delta}$ and $a = \frac{\lambda}{\delta^2}$, respectively.}

\subsection{Half-quadratic majorant}

The previous majorizing approximation is interesting but might lack accuracy, as its curvature does not depend on the tangency point $\vtheta'$. Hereafter, we propose a more sophisticated approximation, reminiscent of the constructions in half-quadratic algorithms for image processing \cite{ALLAIN}.

\begin{proposition}
	\label{prp:sHingeMM}
	For  every $\vtheta' \in \mathbb{R}^{N+1}$, the following function is a tangent majorant of function $\Phi$ involved in Problem~\eqref{eq:sqHingeRef}:
	\begin{equation}
	\label{eq:majorant}
	(\forall \vtheta \in \mathbb{R}^{N+1}) \quad
	h(\vtheta ; \vtheta') =    \Phi(\vtheta') + \nabla    \Phi(\vtheta')^\top (\vtheta - \vtheta') + \frac{1}{2} (\vtheta - \vtheta')^\top \vA(\vtheta')  (\vtheta - \vtheta')
	\end{equation}
	with, 
	\begin{equation}
	(\forall \vtheta \in \mathbb{R}^{N+1}) \quad \label{eq:CurvMat}
	\vA(\vtheta) = 2 \vL^\top \vL + \mathrm{Diag} \left(\left[ \begin{array}{c} \psi(\theta_1) + \eta \\ \vdots \\ \psi(\theta_N)  + \eta \\ \varepsilon \end{array}  \right] \right),
	\end{equation}
	with: $\psi: w \mapsto \varphi'(w)/w$ and $\varepsilon > 0$.
\end{proposition}
For the potential \eqref{eq:l1Smooth},
we have
\begin{equation}
(\forall w \in \mathbb{R}) \quad \psi(w) = \lambda \frac{1}{\sqrt{w^2 + \delta^2}},
\end{equation}
while, for \eqref{eq:l0Smooth},
\begin{equation}
(\forall w \in \mathbb{R}) \quad \psi(w) =  \frac{\lambda}{\delta^2} \exp \left(- \frac{w^2}{2 \delta^2} \right).
\end{equation}

\section{Training SVMs}
\label{sec:MMalg}
In this section, we present a set of MM-based strategies to solve optimization \eqref{eq:sqHingeRef}. First, using the descent lemma majorant, we describe a basic gradient descent algorithm with constant stepsize. Then, using a more sophisticated majorant construction, we derive an MM quadratic approach and provide a skillful strategy for the inversion of the majorant curvature. We also present a subspace acceleration of the aforementioned MM method. Finally, we discuss the stochastic implementation of the training methods and propose a set of hybrid methods with fast convergence in both warm-up and asymptotic regimes.

\subsection{Gradient Descent Approach}
The iterative procedure reads as 
\begin{align}
\label{eq:gradientDescent}
(\forall n \in \mathbb{N})\quad 
\vtheta^{(n+1)} &= \vtheta^{(n)} -\alpha \nabla \Phi(\vtheta^{n)}),
\end{align}
with $\vtheta^{(0)} \in \mathbb{R}^N$. The iterates produced by \eqref{eq:gradientDescent} are guaranteed to converge to a stationary point of \eqref{eq:sqHingeRef} for $\alpha \in ]0,2/\mu[$, where $\mu$ is defined in~\Cref{prp:sHingeLip}. 
\subsection{MM Quadratic Approach}

The gradient descent method is often characterized by a slow convergence. Improved performance can be obtained by the MM quadratic scheme based on~\Cref{prp:sHingeMM}:
\begin{align}
(\forall n \in \mathbb{N})\quad 
\vtheta^{(n+1)} &= \text{argmin}_{\vtheta} \left( \nabla    \Phi(\vtheta^{(n)})^\top (\vtheta - \vtheta^{(n)}) + \frac{1}{2} (\vtheta - \vtheta^{(n)})^\top \vA(\vtheta^{(n)})  (\vtheta - \vtheta^{(n)})\right) \nonumber \\
& = \vtheta^{(n)} -  (\vA(\vtheta^{(n)}))^{-1} \nabla \Phi(\vtheta^{n)}).
\label{eq:fullMM}
\end{align}
The iterative scheme \eqref{eq:fullMM}, related to half-quadratic techniques popular in imaging \cite{ALLAIN}, can be viewed as a preconditioned gradient algorithm. The practical implementation and acceleration of this scheme are discussed below.

\subsubsection{Numerically inverting the majorant curvature}

The computation of the inverse of $\vA(\vtheta^{(n)})$ at each iteration, in \eqref{eq:fullMM}, might be time-consuming.  We propose an approach for computing the product\linebreak $(\vA(\vtheta^{(n)}))^{-1}\nabla\Phi(\vtheta^{(n)})$,  without explicitly constructing the inverse of the matrix. Referring to \eqref{eq:CurvMat}, we majorize the curvature matrix as follows
\begin{equation}
(\forall \vtheta \in \mathbb{R}^{N+1})\quad 
\mathbf{A}(\vtheta) \preceq \mathbf{\overline{A}}(\vtheta) = 2\vL^\top \vL + \AB{\sigma_{\rm max}}(\vtheta)\Id,
\end{equation}
where
\begin{equation} \label{eq:eq:lambdaMax}
\AB{\sigma_{\rm max}}(\vtheta) = \max\left\{\psi(\theta_1) + \eta,\ldots,\psi(\theta_N) + \eta,\varepsilon \right\}.
\end{equation}

Suppose $\vL^\top = \mathbf{QR}$ is the QR factorization of \GF{$\vL^\top \in \mathbb{R}^{(N+1)\times K}$}, where \AB{$\mathbf{Q}$ is an orthogonal matrix of order $N+1$ and $\mathbf{R}$ is a $(N + 1) \times K$ trapezoidal matrix (and hence $\mathbf{R} \mathbf{R}^\top$ is a symmetric matrix of order $N+1$)}. Then 

\begin{equation} \label{eq:A_bar}
(\forall \vtheta \in \mathbb{R}^{N+1})\quad
2 \vL^\top \vL + \AB{\sigma_{\rm max}}(\vtheta)\mathbf{\Id}= 2\mathbf{Q} \mathbf{R} \mathbf{R}^\top \mathbf{Q}^\top + \AB{\sigma_{\rm max}}(\vtheta) \mathbf{\Id}.
\end{equation}

Let $\mathbf{U \Lambda U}^\top$ be the spectral decomposition of $\mathbf{RR}^\top$ where $\mathbf{U}$ is a matrix whose columns are the eigenvectors of $\mathbf{RR}^\top$, and $\mathbf{\Lambda} = \mathrm{Diag}(\lambda_1, \ldots, \lambda_{N+1})$ is diagonal whose elements are the associated eigenvalues. Substituting into \eqref{eq:A_bar}, we obtain 
\begin{equation}
(\forall \vtheta \in \mathbb{R}^{N+1})\quad
2\vL^\top \vL + \AB{\sigma_{\rm max}}(\vtheta)\mathbf{\Id} =  2\mathbf{Q} \mathbf{U}\mathbf{\Lambda} \mathbf{U}^\top \mathbf{Q}^\top + \AB{\sigma_{\rm max}}(\vtheta)\mathbf{\Id}.
\end{equation}

Since $\mathbf{Q}$ and $\mathbf{U}$ have orthogonal columns, considering orthogonal matrix $\mathbf{P} = \mathbf{Q} \mathbf{U}$ yields
\begin{equation} 
(\forall \vtheta \in \mathbb{R}^{N+1})\quad
\label{eq:A_bar_approx}
\mathbf{\overline{A}}(\vtheta)  = \mathbf{P} \hat{\mathbf{\Lambda}}(\vtheta) \mathbf{P}^\top,
\end{equation}
where $\hat{\mathbf{\Lambda}}(\theta)  = 2\mathbf{\Lambda} + \AB{\sigma_{\rm max}}(\vtheta) \mathbf{\Id}$. Consequently, constructing $\mathbf{\overline{A}}(\vtheta)$ as defined in \eqref{eq:A_bar_approx} allows us to compute its inversion efficiently as follows
\begin{equation} \label{eq:A_bar_approx_inv}
(\forall \vtheta \in \mathbb{R}^{N+1})\quad
\mathbf{(\overline{A}(\vtheta))}^{-1} = \mathbf{P} \hat{\mathbf{\Lambda}}(\vtheta)^{-1} \mathbf{P}^\top.
\end{equation}

\begin{remark}
	The proposed approach for computing the inverse of curvature matrix approximation might be even more efficient if $K<N$. Indeed, the ``thin" QR factorization leads to quickly computing the spectral decomposition $\mathbf{U} \mathbf{\Lambda} \mathbf{U}^\top$ of the smaller $K\times K$ matrix $\mathbf{RR}^\top$. 
\end{remark}

\subsubsection{Subspace acceleration}

Another approach for reducing the complexity of \eqref{eq:fullMM}, without jeopardizing its convergence properties, is to resort to a subspace acceleration technique. The method then reads
\begin{align}
\label{eq:subMM}
(\forall n \in \mathbb{N})\quad 
\vtheta^{(n+1)} &= \vtheta^{(n)} - \vD^{(n)} ((\vD^{(n)})^\top \vA(\vtheta^{(n)})\vD^{(n)})^{\dagger} (\vD^{(n)})^\top \nabla \Phi(\vtheta^{n)}).
\end{align}
Hereabove, $\dagger$ stands for the pseudo-inversion, and $\vD^{(n)} \in \mathbb{R}^{(N+1) \times M_n}$ with $M_n \geq 1$ (typically small), is the so-called subspace matrix. A standard choice is
\begin{equation}
\label{eq:3MG}
(\forall n \in \mathbb{N}) \quad \vD^{(n)} = \left[ -\nabla \Phi(\vtheta^{(n)})\; |\; \vtheta^{(n)} - \vtheta^{(n-1)} \right],
\end{equation}
(i.e., $M_n = 2$), with the convention $\vtheta^{(0)} = \mathbf{0}$, which leads to the 3MG (MM Memory Gradient) algorithm. Another 
simplest possibility is 
\begin{equation}
(\forall n \in \mathbb{N}) \quad \vD^{(n)} =  -\nabla \Phi(\vtheta^{(n)}),
\end{equation}
(i.e., $M_n = 1$) which results in a gradient descent technique with varying stepsize
\begin{align}
\label{eq:gradientMM}
(\forall n \in \mathbb{N})\quad
\vtheta^{(n+1)} &= \vtheta^{(n)} - \frac{\nabla \Phi(\vtheta^{n)})^\top \nabla \Phi(\vtheta^{n)})}{\nabla \Phi(\vtheta^{n)})^\top \vA(\vtheta^{(n)}) \nabla \Phi(\vtheta^{n)})} \nabla \Phi(\vtheta^{n)}).
\end{align}
Convergence of the iterates produced by \eqref{eq:subMM} to a stationary point of $\Phi$ is shown in \cite{chouzenoux2013majorize} under mild assumptions. Convergence to the (unique) solution to~\eqref{eq:sqHingeRef} is obtained when we additionally assume that the potential function $\varphi$ is convex and $\eta > 0$. Interestingly, the previously introduced schemes \eqref{eq:gradientDescent} and \eqref{eq:fullMM} can both be viewed as special cases of \eqref{eq:subMM}, and thus inherit the same convergence properties.

\subsubsection{Convergence result}

\textcolor{black}{Let us now state the theoretical convergence guaranties for the MM quadratic method \eqref{eq:fullMM} and its variants.}

\begin{theorem}
	\textcolor{black}{
		Let $\varphi$ given by \eqref{eq:l1Smooth} or \eqref{eq:l0Smooth}. Let $( \vtheta^{(n)})_{n \in \mathbb{N}}$ be generated either by \eqref{eq:fullMM}, or \eqref{eq:subMM}-\eqref{eq:3MG}, or  \eqref{eq:gradientMM}. Then, $( \vtheta^{(n)})_{n \in \mathbb{N}}$ converges to a stationary point of $\Phi$ in \eqref{eq:sqHingeRef}. Moreover, if $\eta>0$ and $\varphi$ is given by \eqref{eq:l1Smooth}, $\Phi$ is strongly convex, and $( \vtheta^{(n)})_{n \in \mathbb{N}}$ converges to its unique minimizer.
	}
\end{theorem}
\begin{proof}
	\textcolor{black}{
		Function $\Phi$ in \eqref{eq:sqHingeRef} is Lipschitz differentiable, by Proposition 1. Moreover, it satisfies Kurdika-\L{}ojasewicz inequality \cite{Attouch11} for $\varphi$ given by  \eqref{eq:l1Smooth} or \eqref{eq:l0Smooth}. 
		The proof results directly from \cite[Theo.3]{chouzenoux2013majorize}, using Proposition 2, and noticing than \eqref{eq:fullMM}, \eqref{eq:subMM}-\eqref{eq:3MG}, and  \eqref{eq:gradientMM}, are all particular cases of an MM quadratic subspace algorithm with one inner iteration.}
\end{proof}

\subsection{Stochastic minimization approaches}
When we face minimization with a particularly large dataset, it is often necessary to use a stochastic technique based on  minibatches \cite{B}. The same issue arises in the context of online learning \cite{onlinebottou} when the entire dataset is not completely available at the beginning of the learning process. Employing a stochastic method may also be convenient for the speed of convergence, especially in a warm-up phase (i.e., first iterations). {The stochastic gradient descent updates the current iterate by a gradient calculated on a single sample $(\vx_k,y_k)$ with randomly chosen $k \in \{1, \ldots,K\}$ in order to lighten the computational cost. For every  $k \in \{1, \ldots,K\}$, let us denote
	\begin{align}
	(\forall \vtheta = [\vw^\top \; \beta]^\top \in \RR^{N+1}) \quad 
	\Phi_k(\vtheta) & = \rho_{\text{hinge}}^2(y_k (\vw^\top \vx_k + \beta)) + f(\vw),\\
	& = \rho_{\text{hinge}}^2(\vL_k^\top \vtheta) + \widetilde{f}(\vtheta),
	\end{align}
	with $\vL_k \in \RR^{N+1}$ as the $k$-th row of $\vL$. We deduce the gradient for the $k$-th sample
	$$
	(\forall \vtheta \in \RR^{N+1}) \quad \nabla  \Phi_k(\vtheta) = \vL_k \max(2  (\vL_k^\top \vtheta -1),0) + \nabla \widetilde{f}(\vtheta).
	$$
} We present in \Cref{alg1} a basic constant stepsize implementation of the stochastic gradient descent method. 
\begin{algorithm}[htpb]
	\caption{Stochastic Gradient (SG) Method}
	\label{alg1}
	\begin{algorithmic}[1]
		\STATE Choose an initial iterate $\vtheta_0$, the stepsize $\alpha>0$ and the maximum number of iterates $maxit$.
		\FOR{$n \in \{0, \dots , maxit\}$}
		\STATE Draw at random an index $\kappa^{(n)} \in \{1,\ldots,K\}$.
		\STATE Compute the stochastic descent direction $\nabla \Phi_{\kappa^{(n)}} (\vtheta^{(n)})$.
		\STATE Set the new iterate as $\theta^{(n+1)} \gets \vtheta^{(n)}-\alpha \nabla \Phi_{\kappa^{(n)}} (\vtheta^{(n)})$
		\ENDFOR
	\end{algorithmic}
\end{algorithm}

In the same fashion, we can also adopt Momentum \cite{loizou} and AdaM \cite{kingma} methods, described in \Cref{momentum} and \Cref{Adam}, respectively.
\begin{algorithm}[H]
	\caption{Momentum}
	\label{momentum}
	\begin{algorithmic}[1]
		\STATE Choose an initial iterate $\vtheta^{(0)}$, the stepsize $\alpha>0$, the maximum number of iterates $maxit$ and $\beta$ $\in [0,1)$
		\STATE Initialize $\vm_0 \leftarrow 0$
		\FOR{$n \in \{1, \dots , maxit\}$}
		\STATE Draw at random an index $\kappa^{(n)} \in \{1,\ldots,K\}$.
		\STATE Compute the stochastic descent direction $\nabla \Phi_{\kappa^{(n)}} (\vtheta^{(n)})$.
		\STATE $\vm^{(n+1)}\leftarrow\beta  \vm^{(n)}+ \nabla \Phi_{\kappa^{(n)}} (\vtheta^{(n)})$
		\STATE $\vtheta^{(n)} \gets \vtheta^{(n-1)}-\alpha \vm^{(n+1)}$
		\ENDFOR
	\end{algorithmic}
\end{algorithm}
\begin{algorithm}[H]
	\caption{Adam}
	\label{Adam}
	\begin{algorithmic}[1]
		\STATE Choose an initial iterate $\vtheta_0$, the stepsize $\alpha>0$, the maximum number of iterates $maxit$, $\hat{\epsilon}$, $\beta_1$ and $\beta_2 \in [0,1)$;
		\STATE Initialize $\vm^{(0)} \leftarrow 0$, $\vv^{(0)} \leftarrow 0$
		\FOR{$n \in \{1, \dots , maxit\}$}
		\STATE Draw at random an index $\kappa^{(n)} \in \{1,\ldots,K\}$.
		\STATE Compute the stochastic descent direction $\nabla \Phi_{\kappa^{(n)}} (\vtheta^{(n)})$.
		\STATE $\vm^{(n)} \gets\beta_1  m^{(n-1)}+(1-\beta_1) \nabla \Phi_{\kappa^{(n)}} (\vtheta^{(n)})$
		\STATE $\vv^{(n)} \gets \beta_2  \vv^{(n-1)}+(1-\beta_2) \nabla \Phi_{\kappa^{(n)}} (\vtheta^{(n)}) \otimes \nabla \Phi_{\kappa^{(n)}} (\vtheta^{(n)}) $
		\STATE $\alpha^{(n)}=\alpha \displaystyle\frac{\sqrt{1-\beta_2^n}}{(1-\beta_1^n)}$
		\STATE $\vtheta^{(n)} \gets \vtheta^{(n+1)}-\alpha^{(n)}  {\vm}^{(n+1)} \oslash(\sqrt{\vv^{(n+1)}}+\hat{\epsilon})$
		\ENDFOR
	\end{algorithmic}
\end{algorithm}
In \Cref{Adam}, $\otimes$ in step 7 denotes the element-wise product, and $\oslash$ in step 9 is the element-wise division. 

{\begin{remark}
		For simplicity, only one index $\kappa^{(n)}$ is selected at each iteration of the above schemes. One may employ the idea of \emph{minibatch}, where a set $\mathcal{B}$ of $B$ indexes is randomly chosen, and the descent direction is given by the weighted sum of the gradients. For example, step 5 in \Cref{alg1} becomes
		$$
		\vtheta^{(n+1)} \gets \vtheta^{(n)}-\alpha \frac1B\sum_{i\in \mathcal{B}}\nabla \Phi_i(\vtheta^{(n)})
		$$
\end{remark}}
Algorithms \ref{alg1}, \ref{momentum} and \ref{Adam} are effective when the hyperparameters are fine-tuned. Several papers in the literature investigate how to develop reliable stepsize selection strategies in stochastic methods, mainly in an adaptive manner \cite{LOD, franchini,LISA}. Details on the hyperparameter
choice will be given in \Cref{sec:numexp}.

\subsubsection{Hybrid approach}

Considering the practical benefits of stochastic methods, and while keeping in mind their convergence-related constraints, we propose to introduce a hybrid strategy. It consists in using stochastic methods to minimize the objective function for a preset $\iota \in \mathbb{N}^*$ number of iterations, and, thus, taking the advantage of their initial learning speed characteristic. After this phase (the \emph{warm-up}), the iterate $\vtheta^{(\iota)}$ obtained from the stochastic methodology is used as an initial point of the deterministic method, benefiting from more stable convergence. Special attention should be paid to the choice of $\iota$, which must result from a trade-off between the benefit offered by the initial speedup of stochastic methods and the convergence properties of deterministic methods. The choice of $\iota$ will also be discussed in \Cref{sec:numexp}. 

\begin{remark}
	From the perspective of convergence guarantees, this warm-up phase has no impact, since it is equivalent to choosing a specific initial point $\vtheta^{(0)}$.
\end{remark}

\section{Numerical Experiments}
\label{sec:numexp}
This section is devoted to numerically assessing the performance of the proposed methods. In particular, we consider three different datasets summarized in \Cref{tab:data}. We split each dataset into training and testing sets, and we use $80\%$ of the elements for the training set and the remaining $20\%$ for the testing set.
\begin{table}[htbp]
	\begin{center}
		
		\begin{tabular}{l|l|c|c}
			dataset & $N+1$ & $K_{\text{training}}$ & $K_{\text{testing}}$ \\ \toprule
			\texttt{a1a}     & 120 & 1284                & 321                \\ \midrule
			\texttt{cina0}   & 133 &       12827          & 3206              \\ \midrule
			\texttt{w8a}     & 301 & 39800                &  9949             \\ \midrule
		\end{tabular}
		
		\caption{Data set characteristics.}
		\label{tab:data}
	\end{center}
\end{table}

The datasets \texttt{a1a} and \texttt{w8a} can be found at \cite{CC01a}, while \texttt{cina0} is available at \cite{Cina0}. We minimize the loss in \cref{eq:sqHingeRef} with three different choices for the regularizer term $\tilde{f}(\vtheta)$, namely $\varphi = 0$ (i.e., squared $l_2$ norm regularization), or $\varphi$ equal to the potential either \eqref{eq:l1Smooth} or \eqref{eq:l0Smooth}. \GF{We emphasize that the case when a squared $l_2$ norm is adopted as the regularizer in addition to the fidelity term in the loss is entirely comparable to the formulation of SVM's primary problem; thus ensuring an experimental comparison with standard SVM as well.}

\begin{remark}[Hyperparameters setting]
	As detailed in \Cref{sec:MMalg}, the choice of proper hyperparameters is crucial for the speed and convergence of stochastic methods. Starting with stochastic methods (SG, Momentum, AdaM) special attention was paid to the choice of the learning rate, which was manually set after an exhaustive search for the optimal one in each method. \GF{All other hyperparameters were chosen as default ones found in the literature.
		In general, $100$ iterations in the deterministic case (or epochs in the stochastic case) appeared enough for all methods. In the case of hybrid method we consider a total of $100$ epochs+iterates. In the hybrid case, we set the warm-up parameter to $\iota=10$, considering it a good trade-off between the speed of stochastic methods and the stability of deterministic ones. In a nutshell, we perform $10$ stochastic epochs for the warm-up phase, and then $90$ deterministic steps (or epochs, which is the same in the deterministic case).}
	
	\noindent Moreover, we empirically set $\eta=10^{-4}$ ans $\lambda=\delta=0$ to experiment only $\ell_2$ regularizer, while $\lambda=\delta=10^{-4}$ and $\eta=0$ is used for other regularizers, which lead to fair performance on all datasets. \GF{A discussion on the influence of $\lambda$ on the results is provided at the end of the section.}
\end{remark}

\subsection{Results}
To test the effectiveness of the methods on the datasets in the \Cref{tab:data}, we will compare the results of various methods associated with different regularizers. In particular, we start by comparing stochastic methods to choose which one is most suitable for the warm-up phase.
\begin{figure}
	\centering
	{\includegraphics[width=0.6\linewidth]{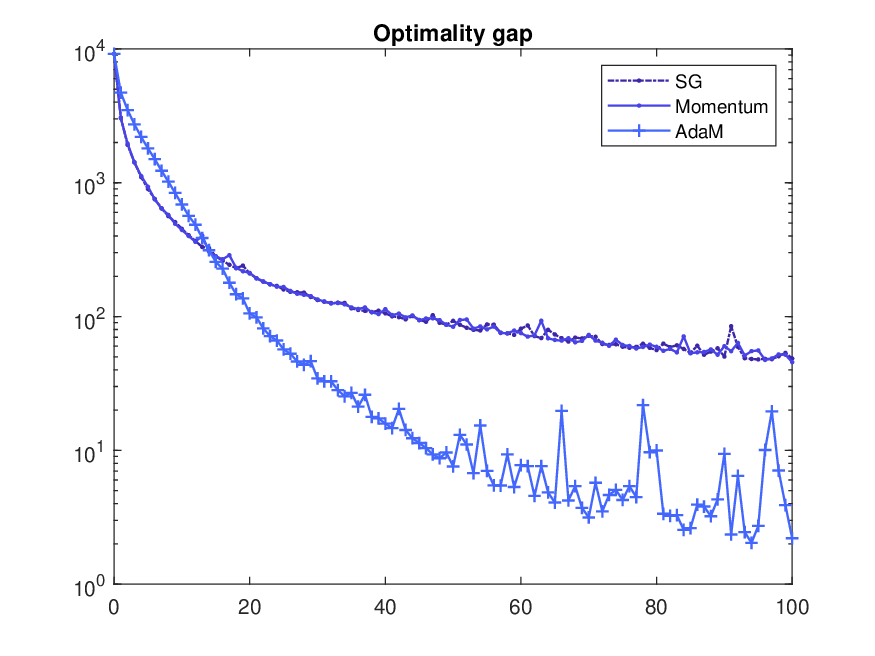}}
	\caption{Stochastic methods with \texttt{a1a} dataset and smooth $l_1$ norm like regularizer.}
	\label{fig:stochastic}
\end{figure}

We define \emph{optimality gap} as the difference between the value of the loss function at the point and the function calculated at an estimate of its minimum. This estimate is derived by letting a deterministic method run for thousands of iterates. 

In \Cref{fig:stochastic} we compare the optimality gap of \texttt{a1a} dataset with smooth $l_1$ norm regularizer. As we can see, the best method at the beginning is AdaM: the behaviour on the other datasets is similar, hence we employ AdaM for the warm-up phase.

\subsubsection{Optimality gap}
\label{sec:optgap}
The methods we are going to compare are as follows:
\begin{itemize}
	\item Gradient descent approach \eqref{eq:gradientDescent} called FULL GRADIENT (FG)
	\item MM quadratic approach \eqref{eq:fullMM} with approximated inverse curvature \eqref{eq:A_bar_approx_inv} called MM INVERSION (MM I)
	\item MM quadratic approach \eqref{eq:fullMM} with approximated inverse curvature \eqref{eq:A_bar_approx_inv} and an initial warm up of $10$ AdaM iterates called HYBRID MM INVERSION (H MM I)
	\item MM quadratic approach \eqref{eq:fullMM} with exact curvature
	\item MM quadratic approach \eqref{eq:fullMM} with exact curvature and an initial warm-up of $10$ AdaM iterates called HYBRID MM (H MM)
	\item MM quadratic with subspace acceleration method \eqref{eq:subMM} called SUBSPACE (SUB)
	\item MM quadratic with subspace acceleration method \eqref{eq:subMM} and an initial warm-up of $10$ AdaM iterates called HYBRID SUBSPACE (H SUB).
\end{itemize}
In the plots of \Cref{fig:optimality}  we consider on the $x$-axis the number of epochs, where an epoch in the deterministic framework means an iterate, whilst in the stochastic framework is a full vision of the dataset. In the $y$-axis we consider the optimality gap.

\begin{figure}\centering
	\newcommand{\factor}{0.45}
	\subfigure[\texttt{a1a} with Welsh potential.]{\includegraphics[width =\factor\textwidth]{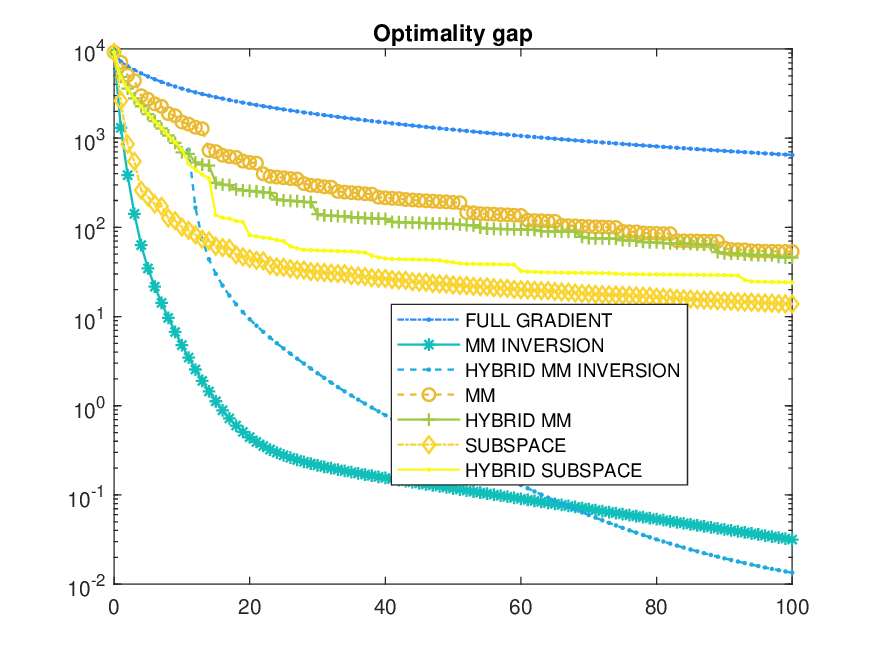}}\hfill\subfigure[\texttt{a1a} with $l_2$ norm only.]{\includegraphics[width =\factor\textwidth]{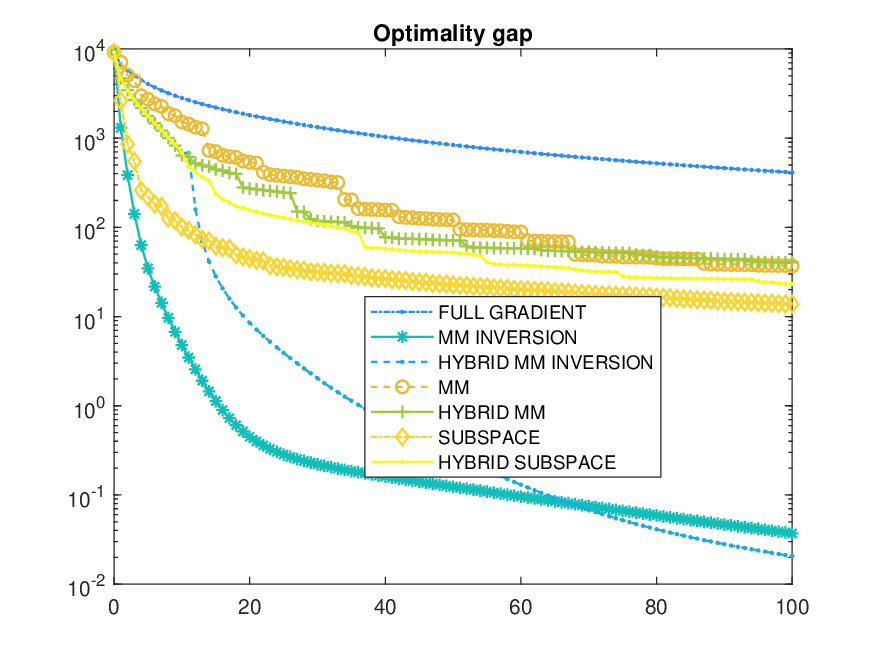}}
	
	\subfigure[\texttt{w8a} with hyperbolic potential.]{\includegraphics[width =\factor\textwidth]{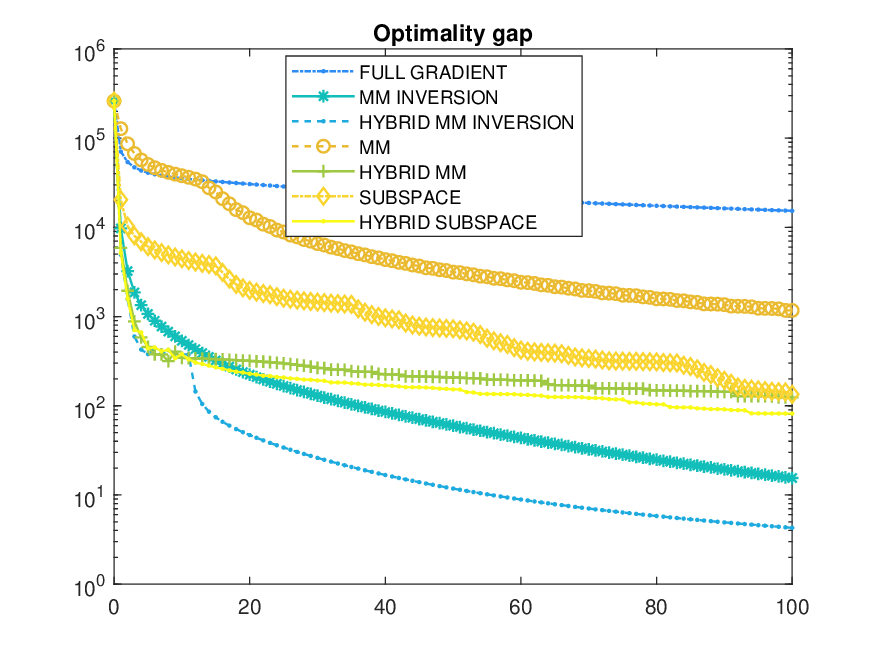}}\hfill\subfigure[\texttt{w8a} with $l_2$ norm only.]{\includegraphics[width =\factor\textwidth]{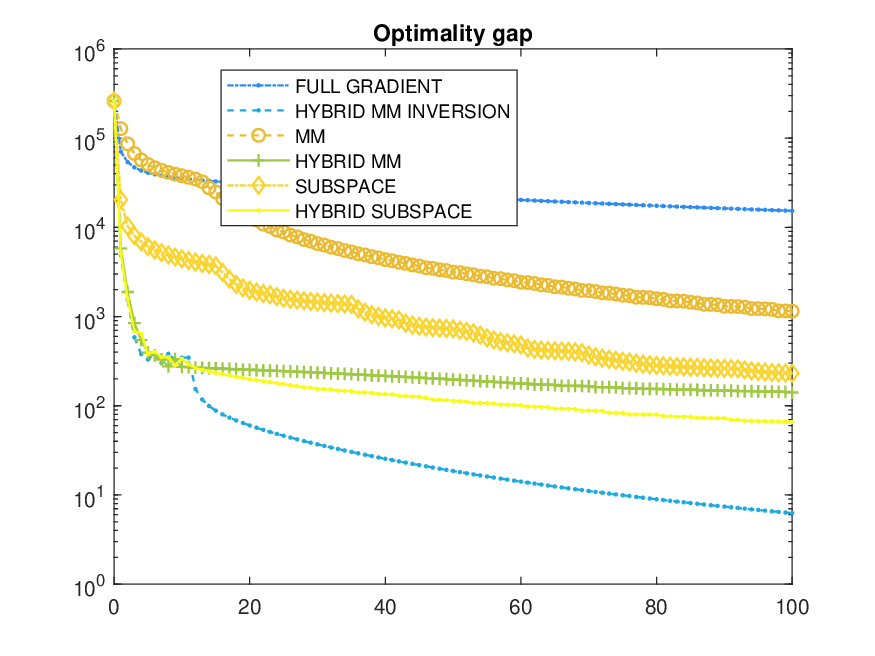}}
	
	\subfigure[\texttt{cina0} with Welsh potential.]{\includegraphics[width =\factor\textwidth]{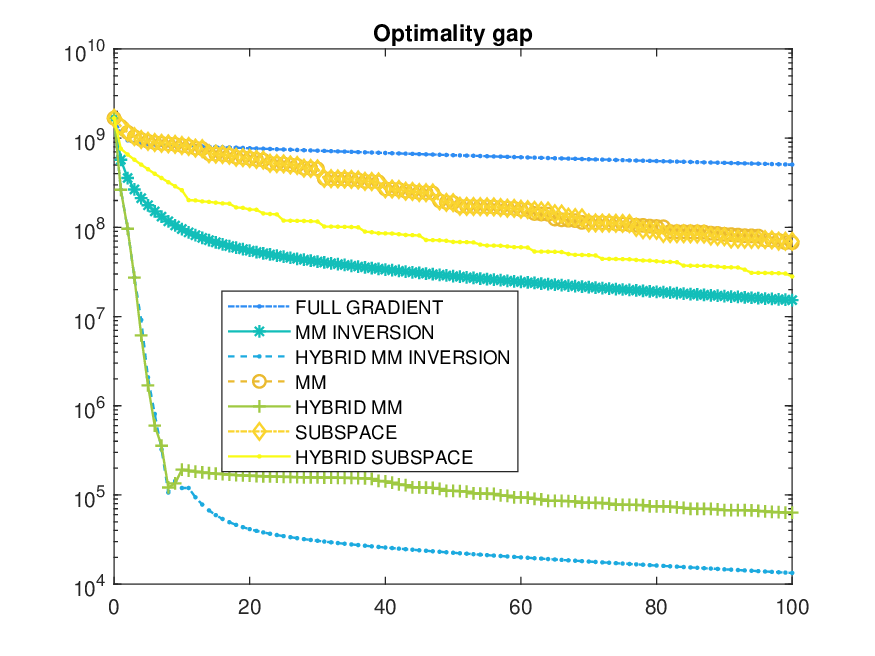}}\hfill\subfigure[\texttt{cina0} with $l_2$ norm only.]{\includegraphics[width =\factor\textwidth]{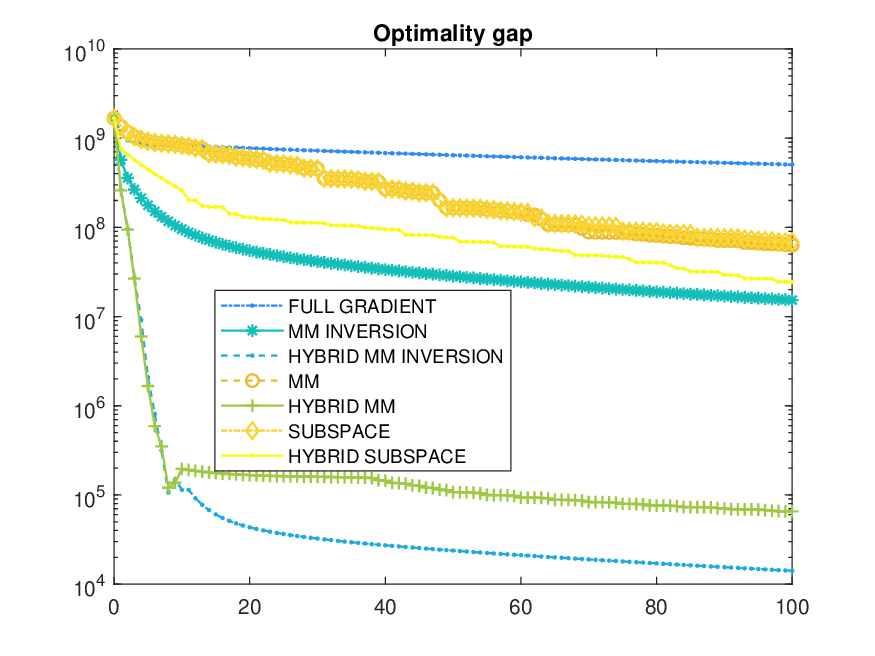}}

	\caption{Optimality gap for different dataset-regularizer combinations. In panel (d) the method MM INVERSION has been removed since it appeared instable.}
	\label{fig:optimality}
\end{figure}

As we can observe from \Cref{fig:optimality}, hybrid methods outperform all the deterministic methods: the warm up strategy pays off even when a low number of warm up iteration $\iota$ is set. Moreover, MM methods seems to exploit the second order information of the functional, allowing an evident boost towards the solution. Among these mixed strategies (MM plus warm up), HYBRID MM INVERSION overcomes all the other, reaching the best optimality gap among all the coupling dataset-regularizer.    

\subsubsection{Performance measure}
In the previous section, we presented results for the optimality gap, calculated on the training set. Here we present performance measures calculated on the test set, to emphasize the generalization ability on unseen data of the proposed method. In this regard, we considered four performance measures well known in the literature: accuracy, precision, recall, and F1-score. In our framework of binary classification, we denote with "positive" and "negative" the two classes, corresponding to the labels 1 and -1 in \cref{eq:linClass}, to be consistent with standard definitions of the previous measures. True Positive (TP) denotes the number of elements of the datasets that are correctly classified as "positive", while True Negative (TN) is the number of elements that are correctly classified as "negative". These two numbers denote the correct classifications. On the other hand, False Positive (FP) and False Negative (FN) denote the elements that are actually negative and are classified as "positive", and vice versa. This terminology stems from classification tasks in medicine and biology. The performance measures considered can be expressed as

\begin{eqnarray*}
	\mbox{Accuracy} &= \displaystyle\frac{TP+TN}{K_{\text{testing}}},\qquad
	\mbox{Precision} &= \displaystyle\frac{TP}{TP+FN},\\\\
	\mbox{Recall} &= \displaystyle\frac{TP}{TP+FN},\qquad
	\mbox{F}_1 \mbox{-score} &= \displaystyle\frac{TP}{TP+\displaystyle\frac{FN+FP}{2}}.
\end{eqnarray*}

\begin{table} [htpb]
	\centering
	\begin{tabular}{l|c|c|c|c|c|c|c|c}
		
		& Reg.  & FG & MMI & H MMI & MM & H MM & SUB & H SUB\\
		\midrule
		\multirow{3}{*}{Accuracy} & {\eqref{eq:l1Smooth}}  & 
		0.7944 & 0.8100 & 0.8100 & 0.8193 & \textbf{0.8255} & 0.8193 & \textbf{0.8255}\\
		& {\eqref{eq:l0Smooth}} & 
		0.7788 & 0.8037 & 0.8100 & 0.8224 & 0.8224 & 0.8193 & 0.8126\\
		& $\times$ & 
		0.7944 & 0.8100 & 0.8100 & 0.8224 & 0.8224 & 0.8193 & 0.8193\\        
		\midrule
		\multirow{3}{*}{Recall} & {\eqref{eq:l1Smooth}}  & 0.6279 & 0.6753 & 0.6753 & 0.6974 & 0.7013 & 0.6974 & \textbf{0.7105}\\
		& {\eqref{eq:l0Smooth}} & 0.6000 & 0.6623 & 0.6753 & 0.7013 & 0.7013 & 0.6974 & 0.6883\\
		& $\times$ & 0.6279 & 0.6753 & 0.6753 & 0.6962 & 0.7067 & 0.6974 & 0.6976 \\
		\midrule
		\multirow{3}{*}{Precision} & {\eqref{eq:l1Smooth}}  & 0.6136 & 0.5909 & 0.5909 & 0.6023 & 0.6136 & 0.6023 & 0.6136\\
		& {\eqref{eq:l0Smooth}} & 0.5795 & 0.5795 & 0.5909 & 0.6136 & 0.6136 & 0.6023 & 0.6023\\
		& $\times$ & 0.6136 & 0.5909 & 0.5909 & \textbf{0.6250}& 0.6023 & 0.6023 & 0.6023 \\
		\midrule
		\multirow{3}{*}{F$_1$} & (8) & 0.6207 & 0.6303 & 0.6303 & 0.6463 & \textbf{0.6585} & 0.6463 & 0.6585\\
		& {\eqref{eq:l0Smooth}} & 0.5896 & 0.6182 & 0.6303 & 0.6545 & 0.6545 & 0.6463 & 0.6424\\
		& $\times$ & 0.6207 & 0.6303 & 0.6303 & 0.6587 & 0.6503 & 0.6463 & 0.6463 \\
		\bottomrule
	\end{tabular}\caption{Performance measures for \texttt{a1a} dataset. \eqref{eq:l1Smooth} and \eqref{eq:l0Smooth}  refers to the choice of the regularization functional, while `$\times$' denotes the case where $\varphi = 0$ (i.e., only $\ell_2$-norm regularization is used). The names of the methods refer to the list presented at the beginning of \Cref{sec:optgap}.}\label{tab:a1a}
\end{table}

\begin{table}[htpb]
	\centering
	\begin{tabular}{l|c|c|c|c|c|c|c|c}
		
		& Reg.  & FG & MMI & H MMI & MM & H MM & SUB & H SUB\\
		\midrule
		\multirow{3}{*}{Accuracy} & {\eqref{eq:l1Smooth}}  & 0.9607 & 0.9916 & \textbf{0.9919} & 0.9876 & 0.9912 & 0.9914 & 0.9916\\
		& {\eqref{eq:l0Smooth}} & 0.9602 & 0.9837 & 0.9918 & 0.9876 & 0.9916 & 0.9911 & 0.9913\\
		& $\times$ & 0.9607 & 0.6471 & 0.9917 & 0.9876 & 0.9912 & 0.9910 & 0.9918\\        
		\midrule
		\multirow{3}{*}{Recall} & (8) & 0.1362 & 0.8000 & \textbf{0.8205} & 0.6299 & 0.7931 & 0.8182 & 0.8051\\
		& (9) & 0.1341 & 0.4837 & 0.8033 & 0.6299 & 0.7813 & 0.8073 & 0.7717\\
		& $\times$ & 0.1362 & 0.0187 & 0.8067 & 0.6299 & 0.7982 & 0.8173 & 0.8136\\        
		\midrule
		\multirow{3}{*}{Precision} & {\eqref{eq:l1Smooth}}  & 0.2821 & 0.6154 & 0.6154 & 0.5128 & 0.5897 & 0.5769 & 0.6090\\
		& {\eqref{eq:l0Smooth}} & 0.2821 & 0.5705 & 0.6282 & 0.5128 & \textbf{0.6410} & 0.5641 & 0.6282\\
		& $\times$ & 0.2821 & 0.4167 & 0.6154 & 0.5128 & 0.5833 & 0.5449 & 0.6154\\        
		\midrule
		\multirow{3}{*}{F$_1$} & (8) & 0.1837 & 0.6957 & 0.7033 & 0.5654 & 0.6765 & 0.6767 & 0.6934\\
		& {\eqref{eq:l0Smooth}} & 0.1818 & 0.5235 & \textbf{0.7050} & 0.5654 & 0.7042 & 0.6642 & 0.6926\\
		& $\times$ & 0.1837 & 0.0357 & 0.6982 & 0.5654 & 0.6741 & 0.6538 & 0.7007\\      
		\bottomrule
	\end{tabular}
	\caption{Performance measures for \texttt{w8a} dataset. \eqref{eq:l1Smooth} and \eqref{eq:l0Smooth}  refers to the choice of the regularization functional, while `$\times$' denotes $\varphi = 0$ (i.e., only $\ell_2$-norm regularization). The names of the methods refer to the list presented at the beginning of \Cref{sec:optgap}.}\label{tab:w8a}
\end{table}

\begin{table}[htbp]
	\centering
	\begin{tabular}{l|c|c|c|c|c|c|c|c}
		
		& Reg.  & FG & MMI & H MMI & MM & H MM & SUB & H SUB\\
		\midrule
		\multirow{3}{*}{Accuracy} & {\eqref{eq:l1Smooth}}  & 0.7748 & 0.9195 & 0.9245 & 0.8531 & 0.9148 & 0.8525 & 0.8668\\
		& {\eqref{eq:l0Smooth}} & 0.7748 & 0.9195 & 0.9245 & 0.8534 & 0.9148 & 0.8531 & 0.8696\\
		& $\times$ & 0.7748 & 0.9195 & \textbf{0.9248} & 0.8534 & 0.9145 & 0.8531 & 0.8677\\ \midrule
		\multirow{3}{*}{Recall} & {\eqref{eq:l1Smooth}}  & 0.5615 & 0.8000 & 0.8608 & 0.7138 & 0.8261 & 0.7121 & 0.7378\\
		& {\eqref{eq:l0Smooth}} & 0.5615 & 0.8501 & 0.8600 & 0.7178 & 0.8253 & 0.7138 & 0.7500\\
		& $\times$ & 0.5615 & 0.8501 & \textbf{0.8610} & 0.7121 & 0.8236 & 0.7133 & 0.7416\\ \midrule
		\multirow{3}{*}{Precision} & {\eqref{eq:l1Smooth}}  & 0.5845 & 0.6154 & 0.8442 & 0.7198 & 0.8490 & 0.7198 & 0.7512\\
		& {\eqref{eq:l0Smooth}} & 0.5845 & 0.8357 & 0.8454 & 0.7126 & 0.8502 & 0.7198 & 0.7428\\
		& $\times$ & 0.5845 & 0.8357 & 0.8454 & 0.7258 & \textbf{0.8514} & 0.7210 & 0.7488 \\
		\midrule
		\multirow{3}{*}{F$_1$} & {\eqref{eq:l1Smooth}}  & 0.5728 & 0.6957 & 0.8524 & 0.7168 & 0.8374 & 0.7159 & 0.7445\\
		& {\eqref{eq:l0Smooth}} & 0.5728 & 0.8429 & 0.8526 & 0.7152 & 0.8376 & 0.7168 & 0.7464\\
		& $\times$ & 0.5828 & 0.8429 & \textbf{0.8531} & 0.7189 & 0.8373 & 0.7171& 0.7452 \\
		\bottomrule
	\end{tabular}
	\caption{Performance measures for \texttt{cina0} dataset. \eqref{eq:l1Smooth} and \eqref{eq:l0Smooth}  refers to the choice of the regularization functional, while `$\times$' denotes $\varphi = 0$ (i.e. only $\ell_2$-norm regularization). The names of the methods refer to the list presented at the beginning of \Cref{sec:optgap}.}\label{tab:cina0}
\end{table}

\begin{table}[htpb]
	\centering
	\textcolor{black}{
		\begin{tabular}{l|c|c|c|c|c|c|c|c}
			& Reg.  & FG & MMI & H MMI & MM & H MM & SUB & H SUB\\
			\midrule
			\multirow{3}{*}{a1a} & {\eqref{eq:l1Smooth}}  & 0.010434 & 0.161774 & 0.031648 & 0.161625 & 0.025913 & 0.036542 & 0.164651\\
			& {\eqref{eq:l0Smooth}} & 0.043637 & 0.186570 & 0.031504 & 0.160993 & 0.034640 & 0.043813 & 0.180505\\
			& $\times$ & 0.173047 & 1.769611 & 0.223276 & 1.680880 & 0.198365 & 0.211096 & 1.738608\\        
			\midrule
			\multirow{3}{*}{cina0} & (8) & 0.181315 & 1.857992 & 0.225515 & 1.808992 & 0.231394 & 0.224193 & 1.888193\\
			& (9) & 0.191336 & 1.986022 & 0.230983 & 2.019854 & 0.255824 & 0.297928 & 1.975059\\
			& $\times$ & 0.173007 & 1.698734 & 0.220615 & 1.702079 & 0.214507 & 0.249245 & 2.063324\\        
			\midrule
			\multirow{3}{*}{w8a} & {\eqref{eq:l1Smooth}}  & 1.098484 & 8.006783 & 1.843458 & 8.389577 & 1.112776 & 1.288889 & 10.018341\\
			& {\eqref{eq:l0Smooth}} & 1.074437 & 11.936826 & 20.957777 & 13.887202 & 1.180779 & 1.133178 & 13.653358\\
			& $\times$ & 1.060243 & 7.085468 & 1.596996 & 8.402317 & 1.279365 & 1.217616 & 8.584693\\             
			\bottomrule
		\end{tabular}
	}
	\caption{Time in second for all the datasets. \eqref{eq:l1Smooth} and \eqref{eq:l0Smooth}  refers to the choice of the regularization functional, while `$\times$' denotes $\varphi = 0$ (i.e., only $\ell_2$-norm regularization). The names of the methods refer to the list presented at the beginning of \Cref{sec:optgap}.}
	\label{tab:time}
\end{table}
\begin{table}[]
	\centering
	\textcolor{black}{
		\begin{tabular}{c|l|l|l|l|l}
			\hline
			$\lambda$            & sparsity & accuracy          & precision         & recall            & F1-score          \\ \hline
			$10^{-1}$    & 55/120       & \bf{0.8162} & \bf{0.6986} & 0.5795 & 0.6335 \\ \hline
			$10^{-2}$   & 35/120       & \bf{0.8162} & \bf{0.68838} & \bf{0.6023} & \bf{0.6424} \\ \hline
			$10^{-3}$  & 21/120       & 0.8100 & 0.6753 & 0.5909 & 0.6303 \\ \hline
			$10^{-4}$ & 18/120       & 0.8100 & 0.6753 & 0.5909 & 0.6303 \\ \hline
			$10^{-5}$ & 16/120       & 0.8100 & 0.6753 & 0.5909 & 0.6303 \\ \hline
		\end{tabular}
	}
	\caption{Sparsity ratio and classification metrics for the a1a dataset with \eqref{eq:l1Smooth} regularization functional, for different $\lambda$ choices.}
	\label{sparsity:a1a}
\end{table}
\begin{table}[]
	\centering
	\textcolor{black}{
		\begin{tabular}{c|l|l|l|l|l}
			\hline
			$\lambda$            & sparsity & accuracy          & precision         & recall            & F1-score          \\ \hline
			$10^{-1}$    & 32/133       & 0.9236 & 0.8577 & 0.8394 & 0.8509 \\ \hline
			$10^{-2}$   & 4/133       & 0.9242 & 0.8625 & 0.8406 & 0.8514 \\ \hline
			$10^{-3}$  & 0/133       & 0.9242 & 0.8571 & 0.8478 & 0.8525 \\ \hline
			$10^{-4}$ & 1/133        & \bf{0.9261} & \bf{0.8635} & 0.8478 & \bf{0.8556} \\ \hline
			$10^{-5}$ & 0/133        & 0.9258 & 0.8606 & \bf{0.8502} & 0.8554 \\ \hline
		\end{tabular}
	}
	\caption{Sparsity ratio and classification metrics for the cina0 dataset with \eqref{eq:l1Smooth} regularization functional, for different $\lambda$ choices.}
	\label{sparsity:cina0}
\end{table}
\begin{table}[]
	\centering
	\textcolor{black}{
		\begin{tabular}{c|l|l|l|l|l}
			\hline
			$\lambda$            & sparsity & accuracy          & precision         & recall            & F1-score          \\ \hline
			$10^{-1}$   & 249/300      & 0.9921 & 0.8469 & 0.6026 & 0.7041 \\ \hline
			$10^{-2}$   & 159/300      & \bf{0.9922} & \bf{0.8482} & 0.609 & \bf{0.709} \\ \hline
			$10^{-3}$  & 44/300      & 0.9921 & 0.8407 & 0.609 & 0.7063 \\ \hline
			$10^{-4}$ & 33/300     & 0.9919 & 0.8205 & \bf{0.6154} & 0.7033 \\ \hline
			$10^{-5}$ & 5/300        & 0.9919 & 0.8205 & \bf{0.6154} & 0.7033 \\ \hline
		\end{tabular}
	}
	\caption{Sparsity ratio and classification metrics for the w8a dataset with \eqref{eq:l1Smooth} regularization functional, for different $\lambda$ choices.}
	\label{sparsity:w8a}
\end{table}

In Tables \ref{tab:a1a}-\ref{tab:cina0} we report the results of the performance measures, each table referring to a different dataset. The second row identifies the type of regularizer. These tables clearly show how hybrid methods provide higher performances when a fixed number of iterations is selected: this confirms that this approach allows to start the deterministic method from a suitable initial point, reaching sooner (with respect to a deterministic method) a reliable estimation of the solution. This behaviour is clear from the plots in \Cref{fig:optimality}. This numerical results hence show that they are able to generalize and that they are particularly effective even on data not seen in the training phase. 

The choice of a suitable regularization functional induces more reliable results, considering all the performance indices: sparse-preserving functions are providing with higher scores among \texttt{a1a} and \texttt{w8a} datasets, while in \texttt{cina0} the $\ell_2$ penalty seems enough to get good results. 

\GF{In Table \ref{tab:time}, we report the computational time in seconds for training the various methods. When we refer to computational time, we mean the entire training phase on all the elements of the dataset, thus referring to the 100 epochs/iterations. All experiments were run on an Intel(R) Core(TM) i7-7700HQ CPU @ 2.80GHz 2.81 GHz. For completeness, all combinations of datasets and regularizers have been reported. The faster training is performed by FG, H MMI, H MM, and SUB approaches. The best compromise between performance and complexity is achieved by hybrid MM methods.}

\GF{We evaluate in Tables \ref{sparsity:a1a}-\ref{sparsity:w8a} the influence of setting parameter $\lambda$, when choosing the convex regularizer \eqref{eq:l1Smooth}. As expected, the sparsity (i.e., number of zero coefficients) in the retrieved coefficients increases with $\lambda$. We can also notice that the best classification metrics are obtained for an intermediary value of $\lambda$. This is in particular the case for a1a and w8a datasets. This emphasizes the important role of the introduced sparsifying penalty.}

\section{Conclusions}
\label{sec:concl}
This paper revisits existing approaches for training Support Vector Machines by considering modern developments around MM strategies. The novel family of proposed optimization methods
address formulations combining a square hinge loss data fidelity function with a smooth sparsity-promoting regularization functional. This combination results in a differentiable objective function, enabling the use of efficient optimization methods for training. The numerical tests performed on three datasets show that the proposed approaches provide reliable results in terms of accuracy, precision, recall, and F1 score and that a hybrid approach integrating some stochastic gradient iterations as a warm up provides an initial boost that leads to a better performance. The results demonstrate that this new approach for training SVMs can be used effectively in a variety of real-world applications, also in a big data context with the joint use of stochastic gradient methods.\\
A natural extension of this work would be to investigate multi-class formulations of SVMs.

\bibliographystyle{plain}
\bibliography{biblio}
\end{document}